\documentclass[letterpaper]{article}
\usepackage{aaai19}
\usepackage{times}
\usepackage{helvet}
\usepackage{courier}
\usepackage[hyphens]{url}
\usepackage{graphicx}
\usepackage{amsmath}
\usepackage{amsfonts}
\usepackage{amssymb}
\usepackage{amsthm}
\usepackage{epsfig}
\usepackage{multirow}
\usepackage{relsize}
\usepackage{stackrel}
\usepackage{bm}
\usepackage{pgf,tikz}
\usetikzlibrary{arrows,automata}
\usepackage[algoruled,longend]{algorithm2e}
\usepackage{colortbl}
\usepackage{tabularx}
\usepackage{breakcites}
\usepackage{hhline}
\title{
    Finding All Bayesian Network Structures within a Factor of Optimal
}
\author{
    Zhenyu A. Liao\textsuperscript{1},
    Charupriya Sharma\textsuperscript{1},
    James Cussens\textsuperscript{2}
    and Peter van Beek\textsuperscript{1}\\~\\
\begin{tabular}{*{2}{>{\centering}p{.5\textwidth}}}
\textsuperscript{1}David R. Cheriton School of Computer Science & \textsuperscript{2}Department of Computer Science \tabularnewline
University of Waterloo & University of York \tabularnewline
Waterloo, ON Canada & York, United Kingdom \tabularnewline
\url{{z6liao,c9sharma,vanbeek}@uwaterloo.ca} & \url{james.cussens@york.ac.uk} 
\end{tabular}
}

\newcolumntype{L}{>{\raggedright\arraybackslash}X}
\newcommand{\yc}{\cellcolor{yellow}}
\newcommand{\opt}{\mathit{OPT}}
\newcommand{\graph}{\mathit{G}}
\newcommand{\graphset}{\mathcal{G}}
\newcommand{\vertex}[1]{V_{#1}}

\newcommand{\vertices}{\mathit{V}}
\newcommand{\edges}{\mathit{E}}

\newcommand{\parents}{\Pi}

\newcommand{\cost}{\mathit{score}}

\newcommand{\problem}{\mathit{\epsilon}\text{BNSL}}

\newcommand{\ceil}[1]{\lceil #1 \rceil}

\newcommand{\varI}{I}

\newcommand{\gobnilp}{GOBNILP}

\newcommand{\score}[2]{\sigma_{#1}({#2})}
\newcommand{\pen}[2]{t_{#1}({#2})}

\newtheorem{theorem}{Theorem}
\newtheorem{innercustomthm}{Theorem}
\newenvironment{customthm}[1]
  {\renewcommand\theinnercustomthm{#1}\innercustomthm}
  {\endinnercustomthm}
\newtheorem{innercustomlem}{Lemma}
\newenvironment{customlem}[1]
  {\renewcommand\theinnercustomlem{#1}\innercustomlem}
  {\endinnercustomlem}
\newtheorem{innercustomcor}{Corollary}
\newenvironment{customcor}[1]
  {\renewcommand\theinnercustomcor{#1}\innercustomcor}
  {\endinnercustomthm}
\newtheorem{corollary}{Corollary}

\newtheorem{lemma}{Lemma}

\newtheorem{definition}{Definition}
\begin{document}

\maketitle

\begin{abstract}
A Bayesian network is  a widely used probabilistic graphical model with
applications in knowledge discovery and prediction.  Learning a
Bayesian network (BN) from data can be cast as an optimization
problem using the well-known score-and-search approach.
However, selecting a single model (i.e., the best scoring BN)
can be misleading or may not achieve the best possible accuracy.
An alternative to committing to a single model is to perform
some form of Bayesian or frequentist model averaging, where
the space of possible BNs is sampled or enumerated in some
fashion. Unfortunately, existing approaches for model averaging
either severely restrict the structure of the Bayesian network
or have only been shown to scale to networks with fewer than 30 random
variables. In this paper, we propose a novel approach to model
averaging inspired by performance guarantees in approximation
algorithms. Our approach has two primary advantages. First,
our approach only considers \emph{credible} models in that they
are optimal or near-optimal in score. Second, our approach
is more efficient and scales to significantly larger
Bayesian networks than existing approaches.
\end{abstract}

%----------------------------------------------------------------------
\section{Introduction}\label{sec:intro}
%----------------------------------------------------------------------

A Bayesian network is a widely used probabilistic graphical
model with applications in knowledge discovery, explanation,
and prediction \cite{Darwiche09,KollerF09}.
A Bayesian network (BN) can be learned from data
using the well-known \emph{score-and-search} approach, where a
scoring function is used to evaluate the fit of a proposed BN
to the data, and the space of directed acyclic graphs (DAGs)
is searched for the best-scoring BN.
However, selecting a single model (i.e., the best-scoring BN)
may not always be the best choice.
When one is using BNs for knowledge discovery and explanation
with limited data, selecting a single model may be
misleading as there may be many other BNs that have scores that
are very close to optimal and the posterior probability of
even the best-scoring BN is often close to zero. As well, when one is using
BNs for prediction, selecting a single model may not achieve
the best possible accuracy. 

An alternative to committing to a single model is to
perform some form of Bayesian or frequentist model averaging 
\cite{Claeskens2008,HoetingMRV1999,KollerF09}.
In the context of knowledge discovery, Bayesian model averaging
allows one to estimate, for example, the posterior probability
that an edge is present, rather than just knowing whether the
edge is present in the best-scoring network.
Previous work has proposed
Bayesian and frequentist model averaging approaches to network structure
learning that enumerate the space of all possible DAGs \cite{KoivistoS04}, sample from the space of all
possible DAGs \cite{HeTW2016,MadiganR1994}, consider
the space of all DAGs consistent with a given ordering of
the random variables \cite{Buntine91,DashC2004},
consider the space of tree-structured or other restricted
DAGs \cite{MadiganR1994,Meila2000}, and consider
only the $k$-best scoring DAGs for some given value of $k$
\cite{ChenCD2015,ChenCD2016,chen2018pruning,ChenT2014,HeTW2016,TianHR10}.
Unfortunately, these existing approaches either severely
restrict the structure of the Bayesian network, such as
only allowing tree-structured networks or only considering a
single ordering, or have only been shown to scale to small
Bayesian networks with fewer than 30 random variables.

In this paper, we propose a novel approach to
model averaging for BN structure learning that is inspired
by performance guarantees in approximation algorithms. Let
$\opt$ be the score of the optimal BN and assume 
without loss of generality that the optimization problem
is to find the minimum-score BN.
Instead of
finding the $k$-best networks for some fixed value of $k$,
we propose to find all Bayesian networks $\mathcal{G}$ that
are within a factor $\rho$ of optimal; i.e.,
\begin{equation}
\label{EQUATION:factor}
\opt \le \cost(\mathcal{G}) \le \rho \cdot \opt,
\end{equation}
for some given value of $\rho \ge 1$, or equivalently,
\begin{equation}
\label{EQUATION:factor_add}
\opt \le \cost(\mathcal{G}) \le \opt+\epsilon,
\end{equation}
for $\epsilon = (\rho - 1) \cdot \opt$. Instead of choosing arbitrary 
values for $\epsilon$, $\epsilon \ge 0$, we show that for the 
two scoring functions BIC/MDL and BDeu, a good choice for the value of $\epsilon$ is 
closely related to the Bayes factor, a model selection criterion 
summarized in \cite{kass1995bayes}.

Our approach has two primary advantages. First, our approach
only considers \emph{credible } models in that they are
optimal or near-optimal in score. Approaches that enumerate
or sample from the space of all possible models consider DAGs
with scores that can be far from optimal; for example,
for the BIC/MDL scoring function the ratio of worst-scoring
to best-scoring network can be four or five orders of
magnitude\footnote{Madigan and Raftery \shortcite{MadiganR1994}
deem such models \emph{discredited} when they make a similar
argument for not considering models whose probability is
greater than a factor from the most probable.}. A
similar but more restricted case can be made against the
approach which finds the $k$-best networks since there is no 
\emph{a priori} way to know how to set the parameter $k$ such that only
credible  networks are considered. Second, and perhaps most
importantly, our approach is significantly more efficient and
scales to Bayesian networks with almost 60 random variables. Existing
methods for finding the optimal Bayesian network structure, e.g., \cite{BartlettC13,vanBeek2015}
rely heavily for their success on a significant body of
pruning rules that remove from consideration many candidate
parent sets both before and during the search. We show that many 
of these pruning rules can be
naturally generalized to preserve the Bayesian networks that are 
within a factor of optimal. 
We modify GOBNILP \cite{BartlettC13}, a state-of-the-art method for finding
an optimal Bayesian network, to implement our generalized pruning rules
and to find all \emph{near}-optimal networks. We show in an experimental
evaluation that the modified GOBNILP scales to significantly
larger networks without resorting to restricting the structure of the Bayesian networks that are learned.

%----------------------------------------------------------------------
\section{Background}\label{sec:background}
%----------------------------------------------------------------------

In this section, we briefly review the necessary background in
Bayesian networks and scoring functions, and define the Bayesian network structure
learning problem (for more background on these topics see \cite{Darwiche09,KollerF09}).

%============================================
\subsection{Bayesian Networks}
%============================================

A Bayesian network (BN) is a probabilistic graphical model that
consists of a labeled directed acyclic graph (DAG), $\graph = (\vertices, \edges)$ in which the vertices $\vertices = \{V_{1}, \ldots, V_{n}\}$
correspond to $n$ random variables, the edges $E$ represent direct
influence of one random variable on another, and each vertex
$\vertex{i}$ is labeled with a conditional probability
distribution $P(V_{i} \mid \Pi_{i})$ that
specifies the dependence of the variable $V_{i}$ on its
set of parents $\Pi_i$ in the DAG. A BN can
alternatively be viewed as a factorized representation of the
joint probability distribution over the random variables and
as an encoding of the Markov condition on the nodes; i.e., given its parents, every variable is conditionally independent of its non-descendents.

Each random variable $V_i$ has state space $\Omega_i$ = $\{v_{i1}$, \ldots , $v_{i{r_i}}\}$, where $r_i$ is the cardinality of $\Omega_i$ and typically $r_i\geq 2$. Each $\parents_{i}$ has state space $\Omega_{\parents_{i}} = \{\pi_{i1},\ldots,\pi_{i{r_{\parents_{i}}}}\}$.  We use $r_{\parents_{i}}$ to refer to the number of possible instantiations of the parent set $\parents_{i}$ of $V_i$ (see Figure \ref{fig:notation}). 
The set $\theta = \{ \theta_{ijk} \}$ for all $ i = \{1,\ldots, n \}, j = \{1,\ldots,r_{\parents_{i}}\}$ and  $k = \{1,\ldots,r_i\}$ represents parameters in $G$ where each element in $\theta$, $\theta_{ijk} = P(v_{ik} \mid \pi_{ij})$.

\begin{figure}[htbp]
\begin{center}
\begin{tikzpicture}[every node/.style={circle, draw, scale=1.2, fill=gray!50}, scale=1.0, rotate = 180, xscale = -1]

\node [fill=white](4) at ( 7,6) {E};
\node [fill=white](5) at ( 9,6) {F};
\node [fill=white](6) at ( 8,5) {B};
\node [fill=white](7) at ( 11,6) {G};
\node [fill=white](8) at ( 13,6) {H};
\node [fill=white](9) at ( 12,5) {C};
\node [fill=white](10) at ( 10,4) {A};

\draw [<-, line width=1pt] (4) -- (6) ;
\draw [<-, line width=1pt] (7) -- (6) ;
\draw [<-, line width=1pt] (5) -- (6) ;
\draw [<-, line width=1pt] (7) -- (9) ;
\draw [<-, line width=1pt] (8) -- (9);
\draw [<-, line width=1pt] (7) -- (9) ;
\draw [<-, line width=1pt] (8) -- (9);
\draw [<-, line width=1pt] (6) -- (10) ;
\draw [<-, line width=1pt] (9) -- (10) ;

\end{tikzpicture}
\caption{Example Bayesian network: Variables $A,B,F$ and $G$ have the state space $\{0,1\}$. The variables $C$ and $E$ have state space $\{0,1,3\} $ and $H$ has state space $\{2,4\}$ 
	Thus $r_A = r_B = r_F = r_G = 2$, $r_C= r_E = 3$ and $r_H = 2$.  Consider the parent set of $G$, $\parents_G = \{B,C\}$ The state space of $\parents_G$ is $\Omega_{\parents_G} = \{ \{0,0\}, \{0,1\}, \{0,3\}, \{1,0\}, \{1,1\}, \{1,3\} \}.$ and $r_{\parents_G} = 6$. }
\label{fig:notation}
\end{center}
\end{figure}
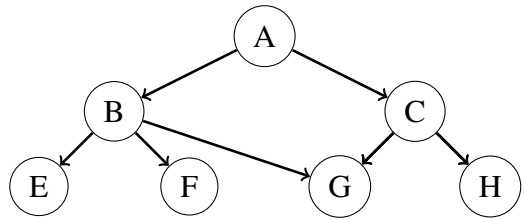

The predominant method for Bayesian network structure learning (BNSL) from data is
the \emph{score-and-search} method.
Let 
%$\graph$ be a DAG over random variables $\vertices$, and let 
$I =
\{I_1, \ldots, I_N\}$ be a dataset where each instance $I_i$ is an $n$-tuple that is a
complete instantiation of the variables in $\vertices$. A \emph{scoring
function} $\sigma( \graph \mid I )$ assigns a real value measuring 
the quality of $\graph=(\vertex,\edges)$ given the data $I$.
Without loss of generality, we assume that a lower score
represents a better quality network structure and omit $I$ when the data is clear from context. 

\begin{definition}
Given a non-negative constant $\epsilon$ and a dataset $I = \{I_1, \ldots, I_N\}$, a \textbf{credible  network} $G$ is a network that has a score $\score{}{\graph}$ such that $\opt \leq \score{}{\graph} \leq \opt + \epsilon$, where $\opt$ is the score of the optimal Bayesian network.  
\end{definition}

In this paper, we focus on solving a problem we call the $\epsilon$-Bayesian Network Structure Learning ($\problem$). Note that the BNSL for the optimal network(s) is a special case of $\problem$ where $\epsilon=0$.
\begin{definition}
Given a non-negative constant $\epsilon$, a dataset $I = \{I_1, \ldots, I_N\}$ 
over random variables $\vertices$ = $\{V_{1}$, \ldots, $V_{n}\}$ and a scoring 
function $\sigma$, the $\epsilon$-Bayesian Network Structure Learning ($\problem$) 
problem is to find all credible networks.
\end{definition}

%============================================
\subsection{Scoring Functions}
%============================================

Scoring functions usually balance goodness of fit to
the data with a penalty term for model complexity
to avoid overfitting. Common scoring functions
include BIC/MDL \cite{LamB94,Schwarz78} and BDeu
\cite{Buntine91,HeckermanGC95}. An important property of these
(and most) scoring functions is decomposability,
where the score of the entire network
$\sigma( \graph )$
can be rewritten as the sum of local scores associated to each vertex
$\sum_{i = 1}^{n} \sigma(\vertex{i},\parents_i )$
that only depends on $\vertex{i}$ and its parent set $\parents_{i}$ in
$\graph$. The local score is abbreviated below as $\score{}{\parents_i}$  when the local node $\vertex{i}$ is clear from context.
%A common assumption is that the
%local score $\sigma_i(\parents_i \mid I )$
%for each candidate parent set
%$\parents_i \subseteq 2^{\vertices - \{\vertex{i}\}}$
%and each random variable $\vertex{i}$ has been computed in a
%preprocessing step prior to the search for the best network
%structure. 
Pruning techniques can be used to reduce the
number of candidate parent sets that need to be considered,
but in the worst-case the number of candidate parent sets
for each variable $\vertex{i}$ is exponential in $n$, where $n$
is the number of vertices in the DAG.

In this work, we focus on the Bayesian Information Criterion (BIC) and the Bayesian Dirichlet, specifically BDeu, scoring functions. The BIC scoring function in this paper is defined as,
\begin{equation*}
    BIC : \score{}{\graph} = \max_{\theta} L_{G,I} (\theta) - t({G})\cdot w.
\end{equation*}
Here, $w = \frac{\log N}{2}$, $t(G)$ is a penalty term and $L_{G,I} (\theta)$ is the log likelihood, given by,
\begin{equation*}
    L_{G,I} (\theta) = \displaystyle  \sum_{i=1}^n \sum_{j=1}^{r_{\Pi_i}} \sum_{k=1}^{r_i} \log \theta_{ijk}^{n_{ijk}} ,
\end{equation*}
where $n_{ijk}$ is the number of instances in $I$ where $v_{ik}$ and $\pi_{ij}$ co-occur.
As the BIC function is decomposable, we can associate a score to $\Pi_i$, a candidate parent set of $\vertex{i}$ as follows,
\begin{equation*}
    BIC : \score{}{\Pi_i} = \max_{\theta_i} L(\theta_i) - t({\Pi_i})\cdot w.
\end{equation*}
Here, $L (\theta_i) =\sum_{j=1}^{r_{\Pi_i}}\sum_{k=1}^{r_i} n_{ijk} \log \theta_{ijk}$ and $t(\Pi_i) = r_{\Pi_i}(r_i -1)$.
The BDeu scoring function in this paper is defined as,
\begin{align*}
    BDeu : \score{}{G} &= \displaystyle  \sum_{i=1}^n \sum_{j=1}^{r_{\Pi_i}} \log \frac{\Gamma (\alpha)}{ \Gamma(\alpha + n_{ij})} \\&+  \sum_{i=1}^n \sum_{j=1}^{r_{\Pi_i}} \sum_{k=1}^{r_i} \log \frac{\Gamma (\frac{\alpha}{r_i}+ n_{ijk})}{\Gamma(\frac{\alpha}{r_i})},
\end{align*}
where $\alpha$ is the equivalent sample size and $n_{ij} = \sum_k n_{ijk}$.
As the BDeu function is decomposable, we can associate a score to $\Pi_i$, a candidate parent set of $\vertex{i}$ as follows,
\begin{align*}
    BDeu : \score{}{\Pi_i} &= \displaystyle  \sum_{j=1}^{r_{\Pi_i}} \Bigg(\log \frac{\Gamma (\alpha)}{ \Gamma(\alpha + n_{ij})} \\&+ \sum_{k=1}^{r_i} \log \frac{\Gamma (\frac{\alpha}{r_i}+ n_{ijk})}{\Gamma(\frac{\alpha}{r_i})} \Bigg).
\end{align*}

%----------------------------------------------------------------------
\section{The Bayes Factor}\label{sec:bf}
%----------------------------------------------------------------------

In this section, we show that a good choice for the value of $\epsilon$ for the
$\problem$ problem is closely related to the 
Bayes factor (BF), a model selection criterion summarized in (Kass and Raftery 1995).

The BF was proposed by Jeffreys as an alternative to significance test \cite{jeffreys1967theory}. It was thoroughly examined as a practical model selection tool in \cite{kass1995bayes}. Let $\graph_0$ and $\graph_1$ be DAGs (BNs) in the set of all DAGs $\graphset$ defined over $\vertex{}$. The BF in the context of BNs is defined as,
$$
BF(\graph_0,\graph_1)=\frac{P(\varI \mid \graph_0)}{P(\varI \mid \graph_1)},
$$
namely the odds of the probability of the data predicted by network $\graph_0$ and $\graph_1$. The actual calculation of the BF often relies on  Bayes' Theorem as follows,
$$
\frac{P(\graph_0 \mid \varI)}{P(\graph_1 \mid \varI)}=\frac{P(\varI \mid \graph_0)}{P(\varI \mid \graph_1)}\cdot\frac{P(\graph_0)}{P(\graph_1)}=\frac{P(\varI,\graph_0)}{P(\varI,\graph_1)}.
$$
Since it is typical to assume the prior over models is uniform in $\problem$, the BF can then be obtained using either $P(\graph \mid \varI)$ or $P(\varI,\graph)\forall\graph\in\graphset$. We use those two representations to show how BIC and BDeu scores relate to the BF.

Using Laplace approximation and other simplifications in \cite{ripley1996pattern}, Ripley derived the following approximation to the logarithm of the marginal likelihood for network $\graph$ (a similar derivation is given in \cite{Claeskens2008}),
\begin{align*}
\log{P(\varI \mid \graph)} =& L_{G,I} (\hat{\theta}) - t({G})\cdot \frac{\log N}{2}+t({G}) \cdot \frac{\log{2\pi}}{2}\\
&-\frac{1}{2}\log{|J_{G,I} (\hat{\theta})|} + \log{P(\hat{\theta} \mid \graph)},
\end{align*}
where $\hat{\theta}$ is the maximum likelihood estimate of model parameters and $J_{G,I} (\hat{\theta})$ is the Hessian matrix evaluated at $\hat{\theta}$. It follows that,
$$
\log{P(\varI \mid \graph)}=-BIC(\varI,\graph)+O(1).
$$
The above equation shows that the BIC score was designed to approximate the log marginal likelihood. 
If we drop the lower-order term, we can then obtain the following equation,
\begin{align*}
    BIC(\varI,\graph_1)-BIC(\varI,\graph_0)&=\log{\frac{P(\varI \mid \graph_0)}{P(\varI \mid \graph_1)}}\\
    &=\log{BF(\graph_0,\graph_1)}.
\end{align*}

It has been indicated in \cite{kass1995bayes} that as $N\rightarrow\infty$, the difference of the two BIC scores, 
dubbed the Schwarz criterion,  approaches the true value of $\log{BF}$ such that,
$$
\frac{BIC(\varI,\graph_1)-BIC(\varI,\graph_0)-\log{BF(\graph_0,\graph_1)}}{\log{BF(\graph_0,\graph_1)}}\rightarrow 0.
$$
Therefore, the difference of two BIC scores can be used as a rough approximation to $\log{BF}$. Note that some papers define BIC to be twice as large as the BIC defined in this paper, but the above relationship still holds albeit with twice the logarithm of the BF.

Similarly, the difference of the BDeu scores can be expressed in terms of the BF. In fact, the BDeu score is the log marginal likelihood where there are Dirichlet distributions over the parameters \cite{Buntine91,HeckermanGC95}; i.e.,
$$
\log{P(\varI,\graph)}=-BDeu(\varI,\graph),
$$
and thus,
\begin{align*}
    BDeu(\varI,\graph_1)-BDeu(\varI,\graph_0)&=\log{\frac{P(\varI,\graph_0)}{P(\varI,\graph_1)}}\\
    &=\log{BF(\graph_0,\graph_1)}.
\end{align*}

The above results are consistent with the observation in \cite{kass1995bayes} that the $\log{BF}$ can be interpreted as a measure for the \emph{relative success} of two models at predicting data, sometimes referred to as the ``weight of evidence'', without assuming either model is true. The desired value of BF, however, is often specific to a study and determined with domain knowledge, e.g., a BF of 1000 is more appropriate in forensic science. \citeauthor{HeckermanGC95}~\shortcite{HeckermanGC95} proposed the following interpreting scale for the BF: a BF of 1 to 3 bears only anecdotal evidence, a BF of 3 to 20 suggests some positive evidence that $\graph_0$ is better, a BF of 20 to 150 suggests strong evidence in favor of $\graph_0$, and a BF greater than 150 indicates very strong evidence. If we deem 20 to be the desired BF in $\problem$, i.e., $\graph_0=\graph^*$ and $\epsilon=\log(20)$, then any network with a score less than $\log(20)$ away from the optimal score would be \emph{credible}, otherwise it would be \emph{discredited}. Note that the ratio of posterior probabilities was defined as $\lambda$ in \cite{TianHR10,ChenT2014} and was used as a metric to assess arbitrary values of $k$ in finding the $k$-best networks.

Finally, the $\problem$ problem using the BIC or BDeu scoring function given a desired BF can be written as,
\begin{align}\label{EQUATION:bf}
    &\opt \le \cost(\graphset) \le \opt+\log{BF}.
\end{align}

%----------------------------------------------------------------------
\section{Pruning Rules for Candidate Parent Sets}\label{sec:pruning}
%----------------------------------------------------------------------

To find all near-optimal BNs given a BF, the local score $\sigma( \parents_{i} )$
for each candidate parent set
$\parents_{i} \subseteq 2^{\vertices - \{\vertex{i}\}}$
and each random variable $\vertex{i}$ must be computed. As this is very cost prohibitive, the search space of candidate parent sets can be pruned, provided that global optimality constraints are not violated. 

%This pruning step is often performed during preprocessing, but can be interleaved with the search for credible  networks.

%Pruning techniques can be used to reduce the number of candidate parent sets that need to be considered.

 %For BIC, we identified a logarithmic bound on the size of parent sets \cite{tian2000branch} and an entropy based rule \cite{CamposC07}. rules\cite{TeyssierK05} gave the simplest pruning rule, which states that if a candidate parent set has a better score than a candidate that is its superset, then such the superset can be safely ignored. This rule works for all decomposable score functions. \cite{CamposJ11} gave pruning rules for BIC/AIC and BDEu that can be applied while the parent sets are being computed. They also gave a $O(\log N)$ bound on the size of parent sets for the BIC scoring function. \cite{suzuki2017branch} improved upon these results, using a new scoring function, called quotient scoring.  also gave a logarithmic bound on parent sets for the MDL scoring function. An entropy based rule for BIC was given by.

% First, we define some of the terms used in these pruning rules. 

A candidate parent set $\Pi_i$ can be \textit{safely pruned} given a non-negative constant $\epsilon \in \mathbb{R}^+$ if $\Pi_i$ cannot be the parent set of $V_i$ in any network in the set of credible networks. Note that for $\epsilon=0$, the set of credible  networks just contains the optimal network(s). We discuss the original rules and their generalization below and proofs for them can be found in the \emph{supplemental material}.

%If a pruning rule omits $\epsilon$, it is assumed that $\epsilon =0$.

\citeauthor{TeyssierK05}~\shortcite{TeyssierK05} gave a pruning rule for all decomposable scoring functions. This rule compares the score of a candidate parent set to those of its subsets. We give a relaxed version of the rule.
% 	\begin{theorem}\cite{TeyssierK05}
% 	Given a vertex variable $\vertex{j}$, and candidate parent sets
% 	$\Pi_j$ and  $\Pi_j^{\prime}$, if $\Pi_j \subset \Pi_j^{\prime}$ and $\score{}{\Pi_j} \leq \score{}{\Pi_j'}$,
% 	$\Pi_j^{\prime}$ can be safely pruned. \label{thm:basicrule}
% 	\end{theorem}
	
%Let us relax this pruning rule.

	\begin{lemma}
	Given a vertex variable $\vertex{j}$, candidate parent sets
	$\Pi_j$ and  $\Pi_j^{\prime}$, and some $\epsilon\in \mathbb{R}^+$,  if $\Pi_j \subset \Pi_j^{\prime}$ and $\score{}{\Pi_j} + \epsilon \geq \score{}{\Pi_j'}$,
	$\Pi_j^{\prime}$ can be safely pruned. \label{lem:scoreprune}
	\end{lemma}
	
%============================================
\subsection{Pruning with BIC/MDL Score}
%============================================

%Another pruning rule can be derived from Theorem \ref{thm:basicrule} \cite{CamposJ11} and applied exclusively to BIC/MDL score. 
% \begin{theorem} \cite{CamposJ11}
% Given a vertex variable $\vertex{j}$, and candidate parent sets,  $\Pi_j$ and $\Pi_j^{\prime}$,
% 	if $\Pi_j \subset \Pi_j^{\prime}$ and $\score{}{\Pi_j} -  \pen{}{\Pi_j'} < 0$,
% 	$\Pi_j^{\prime}$ and all supersets of $\Pi_j^{\prime}$ can be safely pruned  if $\sigma$ is the BIC scoring function.
% \end{theorem}

% This rule can also be relaxed.  
A pruning rule comparing the BIC score and penalty associated to a candidate parent set to those of its subsets was introduced in \cite{CamposJ11}. The following theorem gives a relaxed version of that rule.
\begin{theorem}
Given a vertex variable $\vertex{j}$,  candidate parent sets
	$\Pi_j$ and $\Pi_j'$, and some $\epsilon \in \mathbb{R}^+$,
	if $\Pi_j \subset \Pi_j'$ and $\score{}{\Pi_j} -  \pen{}{\Pi_j'} + \epsilon < 0$,
	$\Pi_j'$ and all supersets of $\Pi_j^{\prime}$ can be safely pruned  if $\sigma$ is the BIC function.
\end{theorem}

Another pruning rule for BIC appears in \cite{CamposJ11}. This provides a bound on the number of possible instantiations of subsets of a candidate parent set. The following theorem relaxes that rule.

% 	\begin{theorem} \cite{CamposJ11}
% 	 Suppose that  we have a variable $V_i$, and a candidate parent set $\Pi_i$ such that $r_{\Pi_i}> \frac{N}{w} \frac{\log r_i}{r_i -1}$.  If $\parents_i \subsetneq \parents_i'$ , then $\parents_i'$ can be safely pruned if $\sigma$ is the BIC scoring function. \label{THEOREM:decamposbic}
% 	\end{theorem}

% 	\begin{corollary}\cite{CamposJ11}
% 	Given a vertex variable $\vertex{i}$ and candidate parent set 
% 	$\Pi_i$, if $\Pi_i$ has more than $\log_2 N$ elements,
% 	$\Pi_i$ can be safely pruned if $\sigma$ is the BIC scoring function. \label{cor:decampossize}
% 	\end{corollary}
	
% Let us relax the pruning rule given in Theorem \ref{THEOREM:decamposbic}.

\begin{theorem}
	Given a vertex variable $V_i$, and a candidate parent set $\Pi_i$ such that $r_{\Pi_i}> \frac{N}{w} \frac{\log r_i}{r_i -1} + \epsilon$ for some $\epsilon \in \mathbb{R}^+$,  if $\parents_i \subsetneq \parents_i'$ , then $\parents_i'$ can be safely pruned if $\sigma$ is the BIC scoring function. \label{THEOREM:decamposbicrelaxed}
\end{theorem}

The following corollary of Theorem \ref{THEOREM:decamposbicrelaxed} gives a useful upper bound on the size of a candidate parent set.

% \newlength\yearposx
% \begin{tikzpicture}[scale=0.35] % timeline 1990-2010->
%     % define coordinates (begin, used, end, arrow)
%     \foreach \x in {}{
%         \pgfmathsetlength\yearposx{(\x-1990)*1cm};
%         \coordinate (y\x)   at (\yearposx,0);
%         \coordinate (y\x t) at (\yearposx,+3pt);
%         \coordinate (y\x b) at (\yearposx,-3pt);
%     }
%     % draw horizontal line with arrow
%     \draw [->] (y1990) -- (y2011);
%     % draw ticks
%   \foreach \x in {1992,2000,2002,2004,2005,2008,2009}
%         \draw (y\x t) -- (y\x b);
%     % annotate
%     \foreach \x in {1992,2002,2005,2009}
%         \node at (y\x) [below=3pt] {\x};
%     \foreach \x in {2000,2004,2008}
%         \node at (y\x) [above=3pt] {\x};

% \end{tikzpicture}

% \begin{tikzpicture}[scale=0.35] % timeline 1990-2010->
%     % define coordinates (begin, used, end, arrow)
%     \foreach \x in {1990,1992,2000,2002,2004,2005,2008,2009,2010,2011}{
%         \pgfmathsetlength\yearposx{(\x-1990)*1cm};
%         \coordinate (y\x)   at (\yearposx,0);
%         \coordinate (y\x t) at (\yearposx,+3pt);
%         \coordinate (y\x b) at (\yearposx,-3pt);
%     }
%     % draw horizontal line with arrow
%     \draw [->] (y1990) -- (y2011);
%     % draw ticks
%   \foreach \x in {1992,2000,2002,2004,2005,2008,2009}
%         \draw (y\x t) -- (y\x b);
%     % annotate
%     \foreach \x in {1992,2002,2005,2009}
%         \node at (y\x) [below=3pt] {\x};
%     \foreach \x in {2000,2004,2008}
%         \node at (y\x) [above=3pt] {\x};

% \end{tikzpicture}

\begin{corollary}
	Given a vertex variable $\vertex{i}$ and candidate parent set 
	$\Pi_i$, if $\Pi_i$ has more than $\ceil{\log_2 N + \epsilon}$ elements, for some $\epsilon \in \mathbb{R}^+$, $\Pi_i$ can be safely pruned if $\sigma$ is the BIC scoring function. \label{cor:decampossizerelaxed}
	\end{corollary}
	
Corollary~\ref{cor:decampossizerelaxed} provides an upper-bound on the size of parent sets based solely on the sample size. The following table summarizes such an upper-bound given different amounts of data $N$ and a BF of 20.
\begin{table}[ht]
    \centering
    \begin{tabular}{@{}c||ccccccc@{}}
     $N$ & $100$ & $500$ & $10^{3}$ & $5\times10^{3}$ & $10^{4}$ & $5\times10^{4}$ & $10^{5}$ \\\hline
     $|\parents|$ & 10 & 12 & 13 & 16 & 17 & 19 & 20
    \end{tabular}
    \label{tab:paLim}
\end{table}

The entropy of a candidate parent set is also a useful measure for pruning. A pruning rule, given by \cite{Campos2017}, provides an upper bound on conditional entropy of candidate parent sets and their subsets. We give a relaxed version of their rule. First, we note that entropy for a vertex variable $\vertex{i}$ is given by,
\begin{align*}
    H(\vertex{i}) &=  -\sum^{r_i}_{k=1}\frac{n_{ik}}{N}\log \frac{n_{ik}}{N} ,
\end{align*}
where $n_{ik}$ represents how many instances in the dataset contain $v_{ik}$, where $v_{ik}$ is an element in the state space $\Omega_i$ of $\vertex{i}$. Similarly, entropy for a candidate parent set $\Pi_i$ is given by,
\begin{align*}
    H(\Pi_i) &=  -\sum^{r_{\Pi_i}}_{j=1}\frac{n_{ij}}{N}\log \frac{n_{ij}}{N} .
\end{align*}
Conditional information is given by,
\begin{equation*}
    H(X \mid Y ) = H(X \cup Y) - H(Y) .
\end{equation*}

% \begin{lemma}\cite{Campos2017}
% Let there be candidate parent sets $\Pi_i$, $\Pi_i'$ for a variable $V_i$ such that $\Pi_i = \Pi_i' \cup \{V_j\}$ for some variable $V_j \notin \Pi_i'$. Then, we have  $L(\Pi_i) - L(\Pi_i') \leq  N \cdot  \min\{H(V_i|\Pi_i'), H(V_j |\Pi_i')\}$. \label{lem:entropy1}
% \end{lemma}

% \begin{theorem}\cite{Campos2017}
% Let $\Pi_i$ be a candidate parent set for variable $V_i$. Let $V_j \notin  \Pi_i$  such that $N \cdot \min \{H(V_i| \Pi_i), H(V_j |\Pi)\} \geq (1 - r_{j}) \cdot t(\Pi_i)$. Then the candidate parent set $\Pi_i' = \Pi_i \cup \{V_j \}$ and all its supersets can be safely pruned.
% \end{theorem}

% We can relax this rule. 

\begin{theorem}
Given a vertex variable $V_i$, and candidate parent set $\Pi_i$, let $V_j \notin \Pi_i$  such that $N \cdot \min \{H(V_i \mid \Pi_i), H(V_j \mid \Pi_i)\} \geq (1 - r_{j}) \cdot t(\Pi_i) +\epsilon$ for some $\epsilon \in \mathbb{R}^+$. Then the candidate parent set $\Pi_i' = \Pi_i \cup \{V_j \}$ and all its supersets can be safely pruned if $\sigma$ is the BIC scoring function.
\end{theorem}

%============================================
\subsection{Pruning with BDeu Score}
%============================================

A pruning rule for the BDeu scoring function appears in \cite{Campos2017} and a more general version is included in \cite{cussens2015gobnilp}. Here, we present a relaxed version of the rule in \cite{cussens2015gobnilp}.  

\begin{theorem}
Given a vertex variable $V_i$ and candidate parent sets $\Pi_i$ and $\Pi_i'$ such that $\Pi_i \subset \Pi_i'$ and $\Pi_i \neq \Pi_i'$, let
$r_i^{+}(\Pi_i')$ be the number of positive counts in the contingency table for $\Pi_i'$. If $\score{}{\Pi_i} + \epsilon < r_i^{+}(\Pi_i') \log r_i$, for some $\epsilon \in \mathbb{R^+}$ then $\Pi_i'$ and the supersets of $\Pi_i'$ can be safely pruned if $\sigma$ is the BDeu scoring function..
\end{theorem}

\begin{table*}[!ht]
    \centering
    \small
    \begin{tabular}{lrr||rrr||rrr||rrr}
        Data & $n$ & $N$ & $T_{3}$ (s) & $|\graphset_{3}|$ & $|\mathcal{M}_{3}|$ & $T_{20}$ (s) & $|\graphset_{20}|$ & $|\mathcal{M}_{20}|$ & $T_{150}$ (s) & $|\graphset_{150}|$ & $|\mathcal{M}_{150}|$ \\
        \hline
tic tac toe & 10 & 958 & 1.9 & 192 & 64 & 2.0 & 192 & 64 & 3.3 & 544 & 160 \\
wine & 14 & 178 & 4.1 & 308 & 51 & 24.9 & 3,449 & 576 & 143.7 & 26,197 & 4,497 \\
adult & 14 & 32,561 & 17.5 & 324 & 162 & 45.1 & 1,140 & 570 & 55.7 & 2,281 & 1,137 \\
nltcs & 16 & 3,236 & 53.8 & 240 & 120 & 201.7 & 1,200 & 600 & 1,005.1 & 4,606 & 2,303 \\
msnbc & 17 & 58,265 & 3,483.0 & 24 & 24 & 7,146.9 & 960 & 504 & 8,821.4 & 1,938 & 1,026 \\
letter & 17 & 20,000 & \yc OT & --- & --- & \yc OT & --- & --- & \yc OT & --- & --- \\
voting & 17 & 435 & 1.3 & 27 & 2 & 4.0 & 441 & 33 & 14.3 & 2,222 & 170 \\
zoo & 17 & 101 & 8.1 & 49 & 13 & 21.9 & 1,111 & 270 & 299.3 & 21,683 & 5,392 \\
hepatitis & 20 & 155 & 7.1 & 580 & 105 & 513.3 & 87,169 & 15,358 & 1,452.8 & 150,000 & 49,269 \\
parkinsons & 23 & 195 & 30.7 & 1,088 & 336 & 3,165.9 & 150,000 & 39,720 & 4,534.3 & 150,000 & 116,206 \\
sensors & 25 & 5456 & \yc OT & --- & --- & \yc OT & --- & --- & \yc OT & --- & --- \\
autos & 26 & 159 & 95.0 & 560 & 200 & 2,382.8 & 50,374 & 17,790 & 6,666.9 & 150,000 & 54,579 \\
insurance & 27 & 1,000 & 49.8 & 8,226 & 2,062 & 244.9 & 104,870 & 25,580 & 414.5 & 148,925 & 36,072 \\
horse & 28 & 300 & 18.8 & 1,643 & 246 & 1,358.8 & 150,000 & 28,186 & 1,962.5 & 150,000 & 69,309 \\
flag & 29 & 194 & 16.1 & 773 & 169 & 4,051.9 & 150,000 & 39,428 & 5,560.9 & 150,000 & 122,185 \\
wdbc & 31 & 569 & 396.1 & 398 & 107 & 10,144.2 & 28,424 & 8,182 & 45,938.2 & 150,000 & 54,846 \\
mildew & 35 & 1000 & 1.2 & 1,026 & 2 & 1.2 & 1,026 & 2 & 2.1 & 2,052 & 4 \\
soybean & 36 & 266 & 7,729.4 & 150,000 & 150,000 & 16,096.8 & 150,000 & 62,704 & 8,893.5 & 150,000 & 118,368 \\
alarm & 37 & 1000 & 6.3 & 1,508 & 122 & 684.2 & 123,352 & 9,323 & 2,258.4 & 150,000 & 8,484 \\
bands & 39 & 277 & 100.9 & 7,092 & 810 & 2,032.6 & 150,000 & 44,899 & 16,974.8 & 150,000 & 95,774 \\
spectf & 45 & 267 & 432.4 & 27,770 & 4,510 & 7,425.2 & 150,000 & 51,871 & 19,664.8 & 150,000 & 63,965 \\
sponge & 45 & 76 & 16.8 & 1,102 & 65 & 1,301.0 & 146,097 & 7,905 & 1,254.4 & 150,000 & 90,005 \\
barley & 48 & 1000 & 0.8 & 182 & 1 & 0.8 & 364 & 2 & 1.3 & 1,274 & 5 \\
hailfinder & 56 & 100 & 171.5 & 150,000 & 20 & 149.4 & 150,000 & 748 & 214.6 & 150,000 & 294 \\
hailfinder & 56 & 500 & 286.1 & 150,000 & 30,720 & 314.1 & 150,000 & 18,432 & 217.3 & 150,000 & 24,576 \\
lung cancer & 57 & 32 & 584.3 & 150,000 & 40,621 & 966.6 & 150,000 & 79,680 & 2,739.7 & 150,000 & 48,236 \\

    \end{tabular}
    \caption{The search time $T$, the number of collected networks $|\graphset|$ and the number of MECs $|\mathcal{M}|$ in the collected networks at BF = 3, 20 and 150 using BIC, where $n$ is the number of random variables in the dataset, $N$ is the number of instances in the dataset and OT = Out of Time.}\label{TAB:bf}
\end{table*}

%----------------------------------------------------------------------
\section{Experimental Evaluation}\label{sec:exp}
%----------------------------------------------------------------------

In this section, we evaluate the proposed BF based method and compare its performance with published $k$-best solvers. 

Our proposed method is more memory efficient comparing to the $k$-best based solvers in BDeu scoring and often collects more networks in a shorter period of time. With the pruning rules generalized above, our method can scale up to datasets with 57 variables in BIC scoring, whereas the previous best results are reported on a network of 29 variables using the $k$-best approach with score pruning \cite{chen2018pruning}.

The datasets are obtained from the UCI Machine Learning Repository \cite{Dua:2017} and the Bayesian Network Repository\footnote{\url{http://www.bnlearn.com/bnrepository/}}. Some of the complete local scoring files are downloaded from the \gobnilp{} website\footnote{\url{https://www.cs.york.ac.uk/aig/sw/gobnilp/#benchmarks}} and are used for the $k$-best related experiments only. Since not all solvers in the $k$-best experiments can take in scoring files, we exclude the time to compute local scores from the comparison. Both BIC/MDL \cite{Schwarz78,LamB94} and BDeu \cite{Buntine91,HeckermanGC95} scoring functions are used where applicable. All experiments are conducted on computers with 2.2 GHz Intel E7-4850V3 processors. Each experiment is limited to 64 GB of memory and 24 hours of CPU time.

%All solvers under consideration are compiled with default makefiles with no explicit options. 

%============================================
\subsection{The Bayes Factor Approach}
%============================================

We modified the development version (9c9f3e6) of GOBNILP, referred below as \gobnilp{\_dev}, to apply pruning rules presented above during scoring and supplied appropriate parameter settings for collecting near-optimal networks\footnote{The modified code is available at: \url{https://www.cs.york.ac.uk/aig/sw/gobnilp/}}. The code is compiled with SCIP 6.0.0 and CPLEX 12.8.0. \gobnilp{} extends the SCIP Optimization Suite
\cite{GleixnerEtal2018OO} by adding a \emph{constraint handler} for
handling the acyclicity constraint for DAGs. If multiple BNs are
required \gobnilp{\_dev} just calls SCIP to ask it to collect feasible
solutions.  In this mode, when SCIP finds a solution, the solution is stored, a
constraint is added to render that solution infeasible and the search
continues. This differs from (and is much more efficient than)
\gobnilp 's current method for finding $k$-best BNs where an entirely
new search is started each time a new BN is found. A recent version of
SCIP has a separate ``reoptimization'' method which might allow better
$k$-best performance for \gobnilp{} but we do not explore that here.
By default when SCIP is asked to collect solutions it turns off all
cutting plane algorithms. This led to very poor \gobnilp{} performance
since \gobnilp{} relies on cutting plane generation. Therefore, this default setting is overridden in \gobnilp{\_dev} to allow cutting planes when collecting solutions.
To find only solutions with objective no worse than ($\opt + \epsilon$), SCIP's
\texttt{SCIPsetObjlimit} function is used. Note that, for efficiency
reasons, this is \textbf{not} effected by adding a linear constraint.

\begin{table*}[!ht]
    \centering
    \small
    \begin{tabular}{lrr||rr||rr||rrr}
        Data & $n$ & $N$ & $T_k$ (s) & $k$ & $T_{EC}$ (s) & $|\graphset_k|$ & $T_{20}$ (s) & $|\graphset_{20}|$ & $|\mathcal{M}_{20}|$\\
        \hline
        \multirow{3}{*}{ tic tac toe } & \multirow{3}{*}{ 10 } & \multirow{3}{*}{ 958 } & 0.2 & 10 & 0.5 & 67 & \multirow{3}{*}{ 0.6 } & \multirow{3}{*}{ 152 } & \multirow{3}{*}{ 24 } \\
& & & 2.8 & 100 & 6.0 & 673 & & & \\
& & & 70.7 & 1,000 & 78.5 & 7,604 & & & \\\hline
\multirow{3}{*}{ wine } & \multirow{3}{*}{ 14 } & \multirow{3}{*}{ 178 } & 3.4 & 10 & 12.0 & 60 & \multirow{3}{*}{ 35.9 } & \multirow{3}{*}{ 8,734 } & \multirow{3}{*}{ 6,262 } \\
& & & 85.0 & 100 & 168.4 & 448 & & & \\
& & & 3,420.4 & 1,000 & 3,064.4 & 4,142 & & & \\\hline
\multirow{3}{*}{ adult } & \multirow{3}{*}{ 14 } & \multirow{3}{*}{ 32,561 } & 3.3 & 10 & 633.5 & 68 & \multirow{3}{*}{ 9.3 } & \multirow{3}{*}{ 792 } & \multirow{3}{*}{ 19 } \\
& & & 73.6 & 100 & 63,328.9 & 1,340 & & & \\
& & & 2,122.8 & 1,000 & \cellcolor{yellow}OT & --- & & & \\\hline
\multirow{3}{*}{ nltcs } & \multirow{3}{*}{ 16 } & \multirow{3}{*}{ 3,236 } & 11.8 & 10 & 47,338.4 & 552 & \multirow{3}{*}{ 125.5 } & \multirow{3}{*}{ 652 } & \multirow{3}{*}{ 326 } \\
& & & 406.6 & 100 & \cellcolor{yellow}OT & --- & & & \\
& & & 13,224.6 & 1,000 & \cellcolor{yellow}OT & --- & & & \\\hline
msnbc & 17 & 58,265 & \cellcolor{yellow}ES & --- & \cellcolor{yellow}ES & --- & 4,018.9 & 24 & 24 \\\hline
\multirow{3}{*}{ letter } & \multirow{3}{*}{ 17 } & \multirow{3}{*}{ 20,000 } & 26.0 & 10 & 18,788.0 & 200 & \multirow{3}{*}{ 56,344.8 } & \multirow{3}{*}{ 20 } & \multirow{3}{*}{ 10 } \\
& & & 909.8 & 100 & \cellcolor{yellow}OT & --- & & & \\
& & & 41,503.9 & 1,000 & \cellcolor{yellow}OT & --- & & & \\\hline
\multirow{3}{*}{ voting } & \multirow{3}{*}{ 17 } & \multirow{3}{*}{ 435 } & 34.1 & 10 & 101.9 & 30 & \multirow{3}{*}{ 6.0 } & \multirow{3}{*}{ 621 } & \multirow{3}{*}{ 207 } \\
& & & 1,125.7 & 100 & 1,829.2 & 3,392 & & & \\
& & & 38,516.2 & 1,000 & 42,415.3 & 3,665 & & & \\\hline
\multirow{3}{*}{ zoo } & \multirow{3}{*}{ 17 } & \multirow{3}{*}{ 101 } & 33.5 & 10 & 99.8 & 52 & \multirow{3}{*}{ 8,418.8 } & \multirow{3}{*}{ 29,073 } & \multirow{3}{*}{ 6,761 } \\
& & & 1,041.7 & 100 & 1,843.4 & 100 & & & \\
& & & 41,412.1 & 1,000 & \cellcolor{yellow}OT & --- & & & \\\hline
\multirow{3}{*}{ hepatitis } & \multirow{3}{*}{ 20 } & \multirow{3}{*}{ 155 } & 351.2 & 10 & 872.3 & 89 & \multirow{3}{*}{ 441.4 } & \multirow{3}{*}{ 28,024 } & \multirow{3}{*}{ 3,534 } \\
& & & 13,560.3 & 100 & 20,244.7 & 842 & & & \\
& & & \cellcolor{yellow}OT & 1,000 & \cellcolor{yellow}OT & --- & & & \\\hline
\multirow{3}{*}{ parkinsons } & \multirow{3}{*}{ 23 } & \multirow{3}{*}{ 195 } & 3,908.2 & 10 & \cellcolor{yellow}OT & --- & \multirow{3}{*}{ 1,515.9 } & \multirow{3}{*}{ 150,000 } & \multirow{3}{*}{ 42,448 } \\
& & & \cellcolor{yellow}OT & 100 & \cellcolor{yellow}OT & --- & & & \\
& & & \cellcolor{yellow}OT & 1,000 & \cellcolor{yellow}OT & --- & & & \\\hline
autos & 26 & 159 & \cellcolor{yellow}OM & 1 & \cellcolor{yellow}OM & --- & \cellcolor{yellow}OT & --- & --- \\
insurance & 27 & 1,000 & \cellcolor{yellow}OM & 1 & \cellcolor{yellow}OM & --- & 8.3 & 1,081 & 133 
    \end{tabular}
    \caption{The search time $T$ and the number of collected networks $k$, $|\graphset_k|$ and $|\graphset_{20}|$ for KBest, KbestEC and GOBNILP\_dev (BF = 20) using BDeu, where $n$ is the number of random variables in the dataset, $N$ is the number of instances in the dataset, OM = Out of Memory, OT = Out of Time and ES = Error in Scoring. Note that $|\graphset_k|$ is the number of DAGs covered by the $k$-best MECs in KBestEC and $|\mathcal{M}_{20}|$ is the number of MECs in the networks collected by GOBNILP\_dev.}\label{TAB:kbest}
\end{table*}

We first use GOBNILP\_dev to find the optimal scores since GOBNILP\_dev takes objective limit ($\opt + \epsilon$) for enumerating feasible networks. Then all networks falling into the limit are collected with a counting limit of 150,000. Finally the collected networks are categorized into Markov equivalence classes (MECs), where two networks belong to the same MEC iff they have the same skeleton and v-structures \cite{VermaP1990}. The proposed approach is tested on datasets with up to 57 variables. The search time $T$, the number of collected networks $|\graphset|$ and the number of MECs $\mathcal{M}$ in the collected networks at BF = 3, 20 and 150 using BIC are reported in Table~\ref{TAB:bf}, where $n$ is the number of random variables in the dataset and $N$ is the number of instances in the dataset. The three thresholds are chosen according to the interpreting scale suggested by \cite{HeckermanGC95} where 3 marks the difference between anecdotal and positive evidence, 20 marks positive and strong evidence and 150 marks strong and very strong evidence. The search time mostly depends on a combined effect of the size of the network, the sample size and the number of MECs at a given BF. Some fairly large networks such as alarm, sponge and barley are solved much faster than smaller networks with a large sample size, e.g., msnbc and letter.

The results also indicate that the number of collected networks and the number of MECs at three BF levels varies substantially across different datasets. In general, datasets with smaller sample sizes tend to have more networks collected at a given BF since near-optimal networks have similar posterior probabilities to the best network. Although the desired level of BF for a study, like the p-value, is often determined with domain knowledge, the proposed approach, given sufficient samples, will produce meaningful results that can be used for further analysis.

%============================================
\subsection{Bayes Factor vs. $k$-Best}
%============================================

In this section, we compare our approach with published solvers that are able to find a subset of top-scoring networks with the given parameter $k$. The solvers under consideration are KBest\_12b\footnote{\url{http://web.cs.iastate.edu/~jtian/Software/UAI-10/KBest.htm}} from \cite{TianHR10}, KBestEC\footnote{\url{http://web.cs.iastate.edu/~jtian/Software/AAAI-14-yetian/KBestEC.htm}} from \cite{ChenT2014}, and GOBNILP 1.6.3
%\footnote{\url{https://www.cs.york.ac.uk/aig/sw/gobnilp/#stable}} 
\cite{BartlettC13}, referred to as KBest, KBestEC and GOBNILP below. The first two solvers are based on the dynamic programming approach introduced in \cite{SilanderM06}. 
%They start by calculating local BDeu scores for all $n2^{n-1}$ parent sets, and then keep only the $k$ top scoring parent sets for each node. The final step differs between those two solvers---KBest uses a best-first search and the fact that every DAG has a sink to find the $k$-best networks, whereas KBestEC goes further with similar ideas to find the $k$-best MECs. 
Due to the lack of support for BIC in KBest and KBestEC, only BDeu with a equivalent sample size of one is used in corresponding experiments.

The most recent stable version of \gobnilp{} is 1.6.3 that works with SCIP 3.2.1. The default configuration is used and experiments are conducted for both BIC and BDeu scoring functions. However, the $k$-best results are omitted here due to its poor performance. Despite that \gobnilp{} can iteratively find the $k$-best networks in descending order by adding linear constraints, the pruning rules designed to find the best network are turned off to preserve sub-optimal networks. In fact, the memory usage often exceeded 64 GB during the initial ILP formulation, indicating that the lack of pruning rules posed serious challenge for GOBNILP. \gobnilp{\_dev}, on the other hand, can take advantage of the pruning rules presented above in the proposed BF approach and its results compare favorably to KBest and KBestEC.

The experimental results of KBest, KBestEC and \gobnilp{\_dev} are reported in Table~\ref{TAB:kbest}, where $n$ is the number of random variables in the dataset, $N$ is the number of instances in the dataset, and $k$ is the number of top scoring networks. The search time $T$ is reported for KBest, KBestEC and GOBNILP\_dev (BF = 20). The number of DAGs covered by the $k$ MECs $|\graphset_k|$ is reported for KBestEC. In comparison, the last two columns are the number of found networks $|\graphset_{20}|$ and the number of MECs $|\mathcal{M}_{20}|$ using the BF approach with a given BF of 20 and BDeu scoring function.

As the number of requested networks $k$ increases, the search time for both KBest and KBestEC grows exponentially. The KBest and KBestEC are designed to solve problems of size fewer than 20\footnote{Obtained through correspondence with the author.}, and so they have some difficulty with larger datasets. They also fail to generate correct scoring files for msnbc. KBestEC seems to successfully expand the coverage of DAGs with some overhead for checking equivalence classes. However, KBestEC took much longer than KBest for some instances, e.g., nltcs and letter, and the number of DAGs covered by the found MECs is inconsistent for nltcs, letter and zoo. The search time for the BF approach is improved over the $k$-best approach except for datasets with very large sample sizes. The generalized pruning rules are very effective in reducing the search space, which then allows GOBNILP\_dev to solve the ILP problem subsequently. Comparing to the improved results in (\citeauthor{ChenCD2015}, \citeyear{ChenCD2015}; \citeyear{ChenCD2016}), our approach can scale to larger networks if the scoring file can be generated.\footnote{We are unable to generate BDeu score files for datasets with over 30 variables.}

\begin{figure}[thb]
    \centering
    \includegraphics[width=\linewidth]{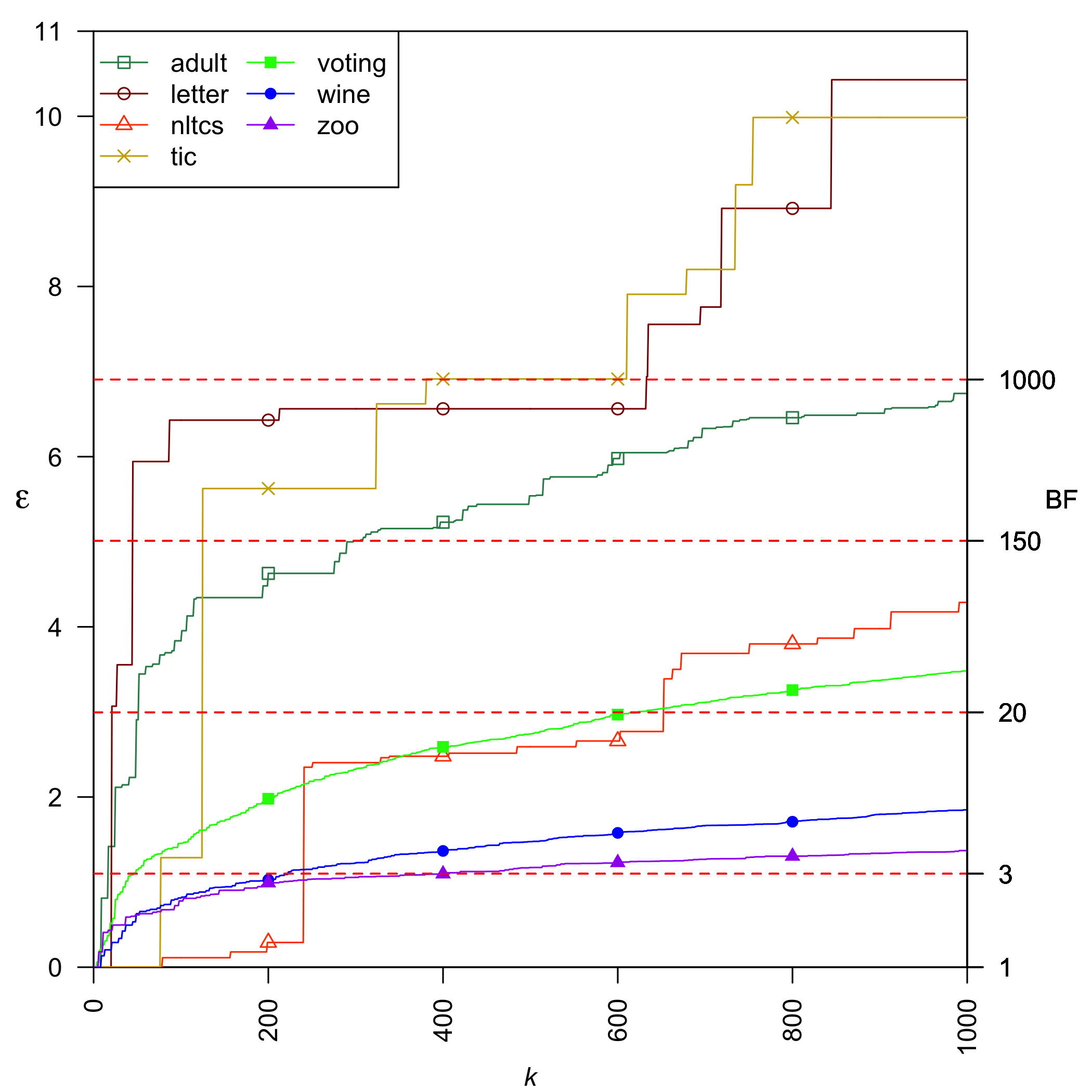}
    \caption{The deviation $\epsilon$ from the optimal BDeu score by $k$ using results from KBest. The corresponding values of the BF ($\epsilon=\log(BF)$, see Equation \ref{EQUATION:bf}) are presented on the right. For example, if the desired BF value is 20, then all networks falling below the dash line at 20 are credible.}\label{fig:diff}
\end{figure}

Now we show that different datasets have distinct score patterns in the top scoring networks. The scores of the 1,000-best networks for some datasets in the KBest experiment are plotted in Figure~\ref{fig:diff}. A specific line for a dataset indicates the deviation $\epsilon$ from the optimal BDeu score by the $k$th-best network. For reference, the red dash lines represent different levels of BFs calculated by $\epsilon=\log{BF}$ (See Equation~\ref{EQUATION:bf}). The figure shows that it is difficult to pick a value for $k$ \emph{a priori} to capture the appropriate set of top scoring networks. For a few datasets such as adult and letter, it only takes fewer than 50 networks to reach a BF of 20, whereas zoo needs more than 10,000 networks. The sample size has a significant effect on the number of networks at a given BF since the lack of data leads to many BNs with similar probabilities. It would be reasonable to choose a large value for $k$ in model averaging when data is scarce and vice versa, but only the BF approach is able to automatically find the appropriate and credible set of networks for further analysis.

%----------------------------------------------------------------------
\section{Conclusion}\label{sec:conclusion}
%----------------------------------------------------------------------

Existing approaches for model averaging for Bayesian network structure learning
either severely restrict the structure of the Bayesian network
or have only been shown to scale to networks with fewer than 30 random
variables. In this paper, we proposed a novel approach to model
averaging inspired by performance guarantees in approximation
algorithms that considers all networks within a factor of optimal. 
Our approach has two primary advantages. First,
our approach only considers \emph{credible} models in that they
are optimal or near-optimal in score. Second, our approach
is significantly more efficient and scales to much larger
Bayesian networks than existing approaches.
We modified GOBNILP \cite{BartlettC13}, a state-of-the-art method for finding
an optimal Bayesian network, to implement our generalized pruning rules
and to find all \emph{near}-optimal networks. Our experimental results demonstrate
that the modified GOBNILP scales to significantly
larger networks without resorting to restricting the structure 
of the Bayesian networks that are learned.

\bibliographystyle{aaai}
\bibliography{csp,probabilistic,scip}

\newpage
\clearpage

\appendix
%----------------------------------------------------------------------
\section{\textit{Supplemental Material}\\Proofs of Pruning Rules}
%----------------------------------------------------------------------

We discuss the original pruning rules and prove their generalization below. A candidate parent set $\Pi_i$ can be \textit{safely pruned} given a non-negative constant $\epsilon \in \mathbb{R}^+$ if $\Pi_i$ cannot be the parent set of $V_i$ in any network in the set of credible networks. Note that proofs of the original rules can be obtained by setting $\epsilon=0$.

%============================================
\subsection{Proof of Lemma \ref{lem:scoreprune}}
%============================================

\citeauthor{TeyssierK05}~\shortcite{TeyssierK05} give a pruning rule that is applicable for all decomposable scoring functions.
\begin{customthm}{1A}\cite{TeyssierK05}
	Given a vertex variable $\vertex{j}$, and candidate parent sets
	$\Pi_j$ and  $\Pi_j^{\prime}$, if $\Pi_j \subset \Pi_j^{\prime}$ and $\score{}{\Pi_j} \leq \score{}{\Pi_j^{\prime}}$,
	$\Pi_j^{\prime}$ can be safely pruned. \label{thm:basicrule}
\end{customthm}

% 	\begin{proof}
% Consider networks $\graph$ and $\graph'$ that are similar except for the parent set of $\vertex{j}$, where $\graph$ has the parent set $\Pi_i$ for $\vertex{j}$ and $\graph'$ has parent set $\Pi_i^{\prime}$ for $\vertex{j}$. 
% \begin{align*}
%     \score{}{\graph} &=  \score{}{\Pi_i} + \sum_{i\neq j} \score{}{\Pi_i}  &[\score{ }{}\text{ is decomposable}]\\
%                     &\leq   \score{}{\Pi_i^{\prime}} + \sum_{i\neq j} \score{}{\Pi_i} &[\text{given}]\\
%                     &=\score{}{\graph'}
% \end{align*}

% Thus, $\graph'$ cannot be optimal.
% 	\end{proof}
Let us relax this pruning rule.
	\begin{lemma}
	Given a vertex variable $\vertex{j}$, candidate parent sets
	$\Pi_j$ and  $\Pi_j^{\prime}$, and some $\epsilon\in \mathbb{R}^+$,  if $\Pi_j \subset \Pi_j^{\prime}$ and $\score{}{\Pi_j} + \epsilon \geq \score{}{\Pi_j^{\prime}}$,
	$\Pi_j^{\prime}$ can be safely pruned. 
	\end{lemma}

	\begin{proof}
Consider networks $\graph$ and $\graph'$ that are the same except for the parent set of $\vertex{j}$, where $\graph$ has the parent set $\Pi_j$ for $\vertex{j}$ and $\graph'$ has the parent set $\Pi_j^{\prime}$ for $\vertex{j}$. 
\begin{align*}
    \score{}{\graph} &=  \score{}{\Pi_j} + \sum_{i\neq j} \score{}{\Pi_j}  &[\score{ }{}\text{ is decomposable}]\\
                    &\leq\score{}{\Pi_j^{\prime}} + \epsilon+ \sum_{i\neq j} \score{}{\Pi_j}  &[\text{given}]\\
                    &=\score{}{\graph'}.
\end{align*}
Thus, $\graph'$ cannot be in the set of credible  networks.
	\end{proof}
	
%============================================
\subsection{Proof of Theorem \ref{thm:decampospar}}
%============================================

An additional pruning rule can be derived from Theorem \ref{thm:basicrule} that is applicable to the BIC/MDL scoring function. 

\begin{customthm}{2A}\cite{CamposJ11}
Given a vertex variable $\vertex{i}$, and candidate parent sets $\Pi_i$ and $\Pi_i^{\prime}$,
	if $\Pi_i \subset \Pi_i^{\prime}$ and $\score{}{\Pi_i} -  \pen{}{\Pi_i^{\prime}} < 0$,
	$\Pi_i^{\prime}$ and all supersets of $\Pi_i^{\prime}$ can be safely pruned  if $\sigma$ is the BIC/MDL scoring function.    
\end{customthm}

% \begin{proof}
%     \begin{align*}
%       &\score{}{\Pi_i} - \pen{}{\Pi_i^{\prime}} < 0  &[\text{given}]\\
%       &\Rightarrow -\score{}{\Pi_i} + \pen{}{\Pi_i^{\prime}} > 0\\
%       &\Rightarrow -\score{}{\Pi_i} + \pen{}{\Pi_i^{\prime}} - L(\Pi_i^{\prime}) > 0 &[L(\Pi_i^{\prime}) > 0]\\
%       &\Rightarrow \score{}{\Pi_i^{\prime}} > \score{}{\Pi_i}
%     \end{align*}
%   By Theorem \ref{thm:basicrule}, $\Pi_i^{\prime}$ cannot be an optimal parent set. Using fact that penalties increase with increase in parent set size, supersets of $\Pi_i^{\prime}$ cannot be optimal. The result follows.
% \end{proof}

Here, $\pen{}{\Pi_i^{\prime}}$ is the penalty term in the BIC scoring function. This pruning rule can also be relaxed. 
\begin{customthm}{2}
Given a vertex variable $\vertex{i}$, candidate parent sets
	$\Pi_i$ and $\Pi_i^{\prime}$, and some $\epsilon \in \mathbb{R}^+$,
	if $\Pi_i \subset \Pi_i^{\prime}$ and $\score{}{\Pi_i} -  \pen{}{\Pi_i^{\prime}} + \epsilon < 0$,
	$\Pi_i^{\prime}$ and all supersets of $\Pi_i^{\prime}$ can be safely pruned if $\sigma$ is the BIC scoring function. \label{thm:decampospar}
\end{customthm}

\begin{proof}
    \begin{align*}
       &\score{}{\Pi_i} -  \pen{}{\Pi_i^{\prime}}  + \epsilon < 0  &[\text{given}]\\
       &\Rightarrow -\score{}{\Pi_i} + \pen{}{\Pi_i^{\prime}} -\epsilon > 0\\
       &\Rightarrow -\score{}{\Pi_i} + \pen{}{\Pi_i^{\prime}} - L(\Pi_i^{\prime})  - \epsilon > 0 &[L(\Pi_i^{\prime}) < 0]\\
       &\Rightarrow \score{}{\Pi_i^{\prime}} > \score{}{\Pi_i} + \epsilon.
    \end{align*}
    By Lemma \ref{lem:scoreprune}, $\Pi_i^{\prime}$ cannot be an optimal parent set. Using the fact that penalties increase with increase in parent set size, supersets of $\Pi_i^{\prime}$ cannot be in the set of credible  networks. The result follows.
\end{proof}
% 	\item
% 	De Campos and Ji \cite{CamposJ11} give a pruning rule that
% 	can be applied \emph{while} the parent sets are being computed.
% 	Assumes the scoring function is BIC. 

%============================================
\subsection{Proof of Theorem \ref{THEOREM:decamposbicrelaxed}}
%============================================

%Another pruning rule for BIC appears in \cite{CamposJ11}.

\begin{customthm}{3A}
\cite{CamposJ11}
	 Given a vertex variable $V_i$ and candidate parent set $\Pi_i$ such that $r_{\Pi_i}> \frac{N}{w} \frac{\log r_i}{r_i -1}$,  if $\parents_i \subsetneq \parents_i'$ , then $\parents_i'$ can be safely pruned if $\sigma$ is the BIC scoring function. \label{THEOREM:decamposbic}
\end{customthm}

\begin{customcor}{3B}
\cite{CamposJ11}
	Given a vertex variable $\vertex{i}$ and candidate parent set 
	$\Pi_i$, if $\Pi_i$ has more than $\log_2 N$ elements,
	$\Pi_i$ can be safely pruned if $\sigma$ is the BIC scoring function. \label{cor:decampossize}
\end{customcor}

% 	\begin{proof}
% 	Assuming $N>4$, take a variable $V_i$ and a parent set $\Pi_i$
% with $|\Pi_i| = \ceil{\log_2 N}$ elements.  Because every variable has at least two states, we know that $r_{\Pi_i} \geq 2^{|\Pi_i|} \geq N> \frac{N}{w}\frac{\log r_i}{r_i -1} $, because $w = \log \frac{N}{2}$ gives us $\frac{\log r_i}{w(r_i -1)} < 1$
% , and by  Theorem \ref{THEOREM:decamposbic} above we know that no proper superset of $\Pi_i$ can
% be an optimal parent set for $\vertex{i}$ as desired.
% 	\end{proof}
	
Let us relax the pruning rule given in Theorem \ref{THEOREM:decamposbic}.

\begin{customthm}{3}
	Given a vertex variable $V_i$, and a candidate parent set $\Pi_i$ such that $r_{\Pi_i}> \frac{N}{w} \frac{\log r_i}{r_i -1} + \epsilon$ for some $\epsilon \in \mathbb{R}^+$,  if $\parents_i \subsetneq \parents_i'$ , then $\parents_i'$ can be safely pruned if $\sigma$ is the BIC scoring function. 
\end{customthm}

\begin{proof}
\begin{align*}
    &\sigma(\Pi_i') - \sigma(\Pi_i)\\
    & \stackbin[]{0}{=} -\displaystyle \max_{\theta_i} L(\Pi_i') + t(\Pi_i')\cdot w + \displaystyle \max_{\theta_i} L(\Pi_i) - t(\Pi_i)\cdot w \\
    & \stackbin[]{1}{\geq} \displaystyle -\max_{\theta_i} L(\Pi_i) + t(\Pi_i')\cdot w  - t(\Pi_i)\cdot w \\
    &\stackbin[]{2}{=} \displaystyle -\sum_{j=1}^{r_{\Pi_i}}n_{ij}(-\sum^{r_i}_{i=1}\frac{n_{ijk}}{n_{ij}}\log \frac{n_{ijk}}{n_{ij}} )+ t(\Pi_i')\cdot w - t(\Pi_i)\cdot w \\
    & \stackbin[]{3}{\geq} \displaystyle -\sum_{j=1}^{r_{\Pi_i}}n_{ij}H(\theta_{ij})- t(\Pi_i')\cdot w  + t(\Pi_i)\cdot w \\
    & \stackbin[]{4}{\geq}  \displaystyle -\sum_{j=1}^{r_{\Pi_i}}n_{ij}\log r_i + r_{\Pi_i}\cdot(r_e -1)\cdot (r_i -1)\cdot w \\
    & \stackbin[]{5}{\geq}  \displaystyle -\sum_{j=1}^{r_{\Pi_i}}n_{ij}\log r_i + r_{\Pi_i}\cdot (r_i -1)\cdot w \\ 
    & \stackbin[]{6}{=}  -N\log r_i + r_{\Pi_i}\cdot (r_i -1)\cdot w\\
    & \stackbin[]{7}{>} \epsilon.
    \end{align*}
    Step 0  uses the definition of $BIC$. Step 1 uses $ \max_{\theta_i} L(\Pi_i')$ is negative. Step 2 uses the fact that the maximum likelihood estimate, $\theta^\ast_{ijk} = \frac{n_{ijk}}{n_{ij}}$ and $n_{ij} = \sum_{i=1}^{r_{i}}n_{ijk}$. Step 3 uses the definition of entropy. Step 4 uses the definition of the penalty function $t$. Step 5 uses $r_e \geq 2$.  Finally, the RHS in Step 6 follows because of the definition of $n_{ij}$. Step 7 uses the assumption of the theorem. 
    
    Using Lemma \ref{lem:scoreprune}, we get the result as desired.  
	\end{proof}
	
%============================================
\subsection{Proof of Corollary \ref{cor:decampossizerelaxed}}
%============================================

Using Theorem \ref{THEOREM:decamposbicrelaxed}, we can get a relaxed version of Corollary \ref{cor:decampossize}.

% \newlength\yearposx
% \begin{tikzpicture}[scale=0.35] % timeline 1990-2010->
%     % define coordinates (begin, used, end, arrow)
%     \foreach \x in {}{
%         \pgfmathsetlength\yearposx{(\x-1990)*1cm};
%         \coordinate (y\x)   at (\yearposx,0);
%         \coordinate (y\x t) at (\yearposx,+3pt);
%         \coordinate (y\x b) at (\yearposx,-3pt);
%     }
%     % draw horizontal line with arrow
%     \draw [->] (y1990) -- (y2011);
%     % draw ticks
%   \foreach \x in {1992,2000,2002,2004,2005,2008,2009}
%         \draw (y\x t) -- (y\x b);
%     % annotate
%     \foreach \x in {1992,2002,2005,2009}
%         \node at (y\x) [below=3pt] {\x};
%     \foreach \x in {2000,2004,2008}
%         \node at (y\x) [above=3pt] {\x};

% \end{tikzpicture}

% \begin{tikzpicture}[scale=0.35] % timeline 1990-2010->
%     % define coordinates (begin, used, end, arrow)
%     \foreach \x in {1990,1992,2000,2002,2004,2005,2008,2009,2010,2011}{
%         \pgfmathsetlength\yearposx{(\x-1990)*1cm};
%         \coordinate (y\x)   at (\yearposx,0);
%         \coordinate (y\x t) at (\yearposx,+3pt);
%         \coordinate (y\x b) at (\yearposx,-3pt);
%     }
%     % draw horizontal line with arrow
%     \draw [->] (y1990) -- (y2011);
%     % draw ticks
%   \foreach \x in {1992,2000,2002,2004,2005,2008,2009}
%         \draw (y\x t) -- (y\x b);
%     % annotate
%     \foreach \x in {1992,2002,2005,2009}
%         \node at (y\x) [below=3pt] {\x};
%     \foreach \x in {2000,2004,2008}
%         \node at (y\x) [above=3pt] {\x};

% \end{tikzpicture}
\begin{customcor}{4}
	Given a vertex variable $\vertex{i}$ and candidate parent set 
	$\Pi_i$, if $\Pi_i$ has more than $\ceil{\log_2 N + \epsilon}$ elements, for some $\epsilon \in \mathbb{R}^+$, $\Pi_i$ can be safely pruned if $\sigma$ is the BIC scoring function.
\end{customcor}

	\begin{proof}
	Assuming $N>4$, take a variable $V_i$ and a parent set $\Pi_i$
with $|\Pi_i| = \ceil{\log_2 N + \epsilon}$ elements.  Because every variable has at least two states, we know that $r_{\Pi_i} \geq 2^{|\Pi_i|}  \geq N +\epsilon > \frac{N}{w}\frac{\log r_i}{r_i -1}+ \epsilon$, because $w =\log \frac{N}{2}$ gives us $\frac{\log r_i}{w(r_i -1)} < 1$ , and by Theorem \ref{THEOREM:decamposbicrelaxed} we know that no proper superset of $\Pi_i$ can be an optimal parent set for $\vertex{i}$ as desired.
	\end{proof} 

%============================================
\subsection{Proof of Theorem \ref{thm:entropyrelaxed}}
%============================================

\begin{customlem}{5A}\cite{Campos2017}
Given a vertex variable $V_i$, and candidate parent sets $\Pi_i$, $\Pi_i'$ such that $\Pi_i = \Pi_i' \cup \{V_j\}$ for some variable $V_j \notin \Pi_i'$, we have  $L(\Pi_i) - L(\Pi_i') \leq  N \cdot  \min\{H(V_i \mid \Pi_i'), H(V_j  \mid \Pi_i')\}$. \label{lem:entropy1}
\end{customlem}

\begin{proof}
First, consider the definition of $L_{i}(\Pi_i)$,
\begin{equation*}
    L(\Pi_i) = \sum_{j=1}^{r_{\Pi_i}}\sum_{k=1}^{r_i} n_{ijk} \log {\theta}_{ijk}^{\ast} ,
\end{equation*}
where ${\theta}_{ijk}^{\ast}$ is the maximum likelihood estimate. This gives us $N\cdot H(V_i \mid \Pi_i) = -L(\Pi_i)$. Thus, we get,
\begin{align*}
    L(\Pi_i) - L(\Pi_i') &= N\cdot (H(V_i \mid  \Pi_i') - H(V_i \mid  \Pi_i))\\
                            &\stackbin[]{1}{\leq} N\cdot H(V_i \mid  \Pi_i') .
\end{align*}
We use the fact that entropy is positive. Now, consider the definition of mutual information,
\begin{align*}
    I(X,Y \mid Z) = H(X \mid Z) = H(X \mid Y \cup Z).
\end{align*}
This gives us,
\begin{align*}
    L(\Pi_i) - L(\Pi_i') &= N \cdot I(V_i, V_j  \mid  \Pi_i')\\
                        & \stackbin[]{2}{=} N\cdot (H(V_j \mid  \Pi_i') - H(V_j \mid \Pi_i' \cup \{V_i\}))\\
    \Rightarrow L(\Pi_i) - L(\Pi_i') &\stackbin[]{3}{\leq} N \cdot  \min\{H(V_i \mid \Pi_i'), H(V_j  \mid \Pi_i')\}.
\end{align*}
Step 3 combines Steps 1 and 2. The result follows as desired.
\end{proof}

\begin{customthm}{5B}
\cite{Campos2017}
Given a vertex variable $V_i$, and candidate parent set $\Pi_i$, let $V_j \notin \Pi_i$  such that $N \cdot \min \{H(V_i \mid  \Pi_i), H(V_j  \mid \Pi_i)\} \geq (1 - r_{j}) \cdot t(\Pi_i)$. Then the candidate parent set $\Pi_i' = \Pi_i \cup \{V_j \}$ and all its supersets can be safely pruned if $\sigma$ is the BIC scoring function. 
\end{customthm}

% \begin{proof}
% \begin{align*}
%     \score{}{\Pi_i'} 
%                     &\stackbin[]{0}{=} -L(\Pi_i')  +  t(\Pi_i')\\
%                     & \stackbin[]{1}{\geq} -L(\Pi_i) - N \cdot \min\{H(V_i|\Pi),H(X_j |\Pi)\} + t(\Pi_i')\\
%                     &\stackbin[]{2}{\geq} -L(\Pi_i) + (1 - r_j)t(\Pi_i) +t(\Pi_i') \\
%                     &\stackbin[]{3}{=} -L(\Pi_i) + t(\Pi_i) - r_j \cdot t(\Pi_i) + t(\Pi_i') \\
%                     &\stackbin[]{4}{=} -L(\Pi_i) + t(\Pi_i) - r_j \cdot r_{\Pi_i} \cdot (r_i -1) + t(\Pi_i') \\
%                     &\stackbin[]{5}{=} -L(\Pi_i) + t(\Pi_i) - t(\Pi_i') + t(\Pi_i') \\
%                     &\stackbin[]{6}{=} \score{}{\Pi_i}.
% \end{align*}
% Step 1 uses Lemma \ref{lem:entropy1}. Step 2 uses the assumptions of the question. Step 4 uses the definition of $t$. Step 5 uses $\Pi_i' = \Pi_i \cup \{V_j \}$. The result follows as desired.
% \end{proof}
We can relax this rule.

\begin{customthm}{5}
Given a vertex variable $V_i$, and candidate parent set $\Pi_i$, let $V_j \notin \Pi_i$  such that $N \cdot \min \{H(V_i \mid  \Pi_i), H(V_j  \mid \Pi_i)\} \geq (1 - r_{j}) \cdot t(\Pi_i) +\epsilon$ for some $\epsilon \in \mathbb{R}^+$ . Then the candidate parent set $\Pi_i' = \Pi_i \cup \{V_j \}$ and all its supersets can be safely pruned if $\sigma$ is the BIC scoring function. \label{thm:entropyrelaxed}
\end{customthm}

\begin{proof}
\begin{align*}
    \score{}{\Pi_i'} 
                    &\stackbin[]{0}{=} -L(\Pi_i')  +  t(\Pi_i')\\
                    &\stackbin[]{1}{\geq} -L(\Pi_i) - N \cdot \min\{H(V_i \mid \Pi_i); H(V_j \mid \Pi_i)\} + t(\Pi_i')\\
                    &\stackbin[]{2}{\geq} -L(\Pi_i) + (1 - r_j)\cdot t(\Pi_i) +\epsilon  +t(\Pi_i') \\
                    &\stackbin[]{3}{=} -L(\Pi_i) + t(\Pi_i) - r_j \cdot t(\Pi_i) + \epsilon + t(\Pi_i') \\
                    &\stackbin[]{4}{=} -L(\Pi_i) + t(\Pi_i) - r_j \cdot r_{\Pi_i} \cdot (r_i -1) + \epsilon + t(\Pi_i') \\
                    &\stackbin[]{5}{=} -L(\Pi_i) + t(\Pi_i) - t(\Pi_i') + \epsilon + t(\Pi_i') \\
                    &\stackbin[]{6}{=} \score{}{\Pi_i} + \epsilon.
\end{align*}
Step 1 uses Lemma \ref{lem:entropy1}. Step 2 uses the assumptions of the question. Step 4 uses the definition of $t$. Step 5 uses $\Pi_i' = \Pi_i \cup \{V_j \}$. Using Lemma \ref{lem:scoreprune}, the result follows as desired.
\end{proof}

%============================================
\subsection{Proof of Theorem \ref{thm:jamesrelaxed}}
%============================================

%Now, we show a pruning rule for the BDeu scoring function.
\begin{customlem}{6A}
Let $n_{ij}$ be a positive integer and $\alpha'$ be a positive real number. Then
\begin{equation*}
    \log \frac{\Gamma(n_{ij} + \alpha')}{\Gamma(\alpha')} = \sum_{i=0}^{n_{ij -1}}\log (i + \alpha')
\end{equation*} \label{lem:gamma}
\end{customlem}

\begin{proof}
We start with the property that $\Gamma(x+1) = x\Gamma(x)$ for any positive real number $x$. As $\alpha' > 0$, this gives us,
\begin{align*}
    \frac{\Gamma(1 + \alpha')}{\Gamma(\alpha') } &\stackbin[]{0}{=} \alpha' \\
     \frac{\Gamma(2 + \alpha')}{ \Gamma(1+ \alpha')} &\stackbin[]{1}{=} (1+ \alpha') \\
    \Rightarrow    \frac{\Gamma(1 + \alpha')\cdot\Gamma(2 + \alpha')}{ \Gamma(1+ \alpha')\Gamma(\alpha') }  &\stackbin[]{2}{=} \alpha' (1+ \alpha')\\
    \Rightarrow  \frac{\Gamma(1 + \alpha')\cdots\Gamma(n_{ij
    }+ \alpha')}{ \Gamma(n_{ij} -1 + \alpha')\cdots\Gamma(\alpha')} &\stackbin[]{3}{=} \alpha'\cdots (n_{ij} -1 + \alpha') \\
    \Rightarrow  \frac{\Gamma(n_{ij} + \alpha') }{ \Gamma(\alpha')} &\stackbin[]{4}{=} \alpha'\cdots (n_{ij} -1 + \alpha')\\
    \Rightarrow \log \frac{\Gamma(n_{ij}+ \alpha')}{\Gamma(\alpha')} &\stackbin[]{5}{=} \sum_{i=0}^{n_{ij -1}}\log (i + \alpha').
\end{align*}
Step 1 uses $1 + \alpha'$. Step 2 follows by multiplication of the equations in Step 1 and Step 0. Step 3 follows by repeated application of the identity. Step 4 cancels identical terms in the LHS. The result follows as desired.
\end{proof}

\begin{customlem}{6B}
Let $\{n_{ijk}\}_{k=1,...r_{i}}$ be non-negative integers with a positive sum, $n_{ij} = \sum_{k=1}^{r_i} n_{ijk}$ and $\alpha''$ be a positive real number. Then
\begin{equation*}
    \sum_{k=1}^{r_i} \log \frac{\Gamma(n_{ijk} + \alpha'')}{\Gamma(\alpha'')} \leq \log \frac{\Gamma(n_{ij} + \alpha'')}{\Gamma(\alpha'')}
\end{equation*} \label{lem:gamma2}
\end{customlem}

\begin{proof}
Consider allocation of $\{n_{ijk}\}_{k=1,...,{r_{i}}}$  items over the $r_i$ bins. There are two cases.
\begin{itemize}
    \item Let there be some index $k^\ast$ such that $n_{ijk^{\ast}} = n_{ij}$. This means that $n_{ijk}=0$ for all $k\neq k^\ast$. It follows that $\sum_{k=1}^{r_i} \log \frac{\Gamma(n_{ijk} + \alpha'')}{\Gamma(\alpha'')} = \log \frac{\Gamma(n_{ij} + \alpha'')}{\Gamma(\alpha'')}$.
    \item Let there be two indices $k_1$ and $k_2$ such that $n_{ijk_{1}} > 0$ and $n_{ijk_{2}} > 0$. Without loss of generality, we can assume that $n_{ijk_{1}} \geq n_{ijk_{2}}$. We move one item from bin $k_1$ to bin $k_2$. The sum $n_{ij}$ remains constant. By Lemma \ref{lem:gamma}, an increase in the RHS by $\log (n_{ijk_{1}}  + \alpha'') - \log (n_{ijk_{2}}  -1+ \alpha'')$, results in a corresponding increase in the LHS. Note that the assumption $n_{ijk_{1}} \geq n_{ijk_{2}}$ means that this increase is positive. 
     By increasing counts at the expense of small counts in  this way a sequence of distributions of the fixed sum $n_{ij}$ over the $r_i$ bins can be constructed for which the LHS of Lemma \ref{lem:gamma2} is increasing. The sequence terminates when $n_{ijk^{\ast}} = n_{ij}$ for some $k^\ast$. The result follows.
\end{itemize} 

\end{proof}

\begin{customthm}{6C}
\cite{cussens2015gobnilp}
\begin{align*}
     &\sum_{j=1}^{r_{\Pi_i}} \Bigg( \frac{\Gamma( \alpha')}{\Gamma(n_{ij}+\alpha')} +  \sum_{k=1}^{r_i} \log \frac{\Gamma(n_{ijk} + \frac{\alpha'}{r_i})}{\Gamma( \frac{\alpha'}{r_i})}\Bigg) \\ 
     &\leq  \sum_{i=0, j:n_{ij}>0}^{n_{ij}} \log \Big( \frac{i+ a'/r_i}{i+\alpha}\Big).
\end{align*} \label{thm:gamma4}
\end{customthm}

\begin{proof}
 \begin{align*}
     &\sum_{j=1}^{r_{\Pi_i}} \Bigg( \log \frac{\Gamma( \alpha')}{\Gamma(n_{ij}+\alpha')} +  \sum_{k=1}^{r_i} \log \frac{\Gamma(n_{ijk} + \frac{\alpha'}{r_i})}{\Gamma( \frac{\alpha'}{r_i})}\Bigg)\\
      &\stackbin[]{1}{\leq} \sum_{j=1}^{r_{\Pi_i}}  \Bigg(\log \frac{\Gamma( \alpha')}{\Gamma(n_{ij}+ \alpha')} + \log \frac{\Gamma(n_{ij} + \frac{\alpha'}{r_i})}{\Gamma( \frac{\alpha'}{r_i})}\Bigg) \\
     &\stackbin[]{2}{\leq} \sum_{j=1}^{r_{\Pi_i}}  \Bigg(\log \frac{\Gamma( \alpha')}{\Gamma(n_{ij}+ \alpha')} \frac{\Gamma(n_{ij} + \frac{\alpha'}{r_i})}{\Gamma( \frac{\alpha'}{r_i})}\Bigg) \\
     & \stackbin[]{3}{\leq} \sum_{i=0, j:n_{ij}>0}^{n_{ij}-1} \Big( \log  \frac{i+ a'/r_i}{i+\alpha'}\Big)  \\
     &\stackbin[]{4}{\leq} \sum_{i=0, j:n_{ij}>0}^{n_{ij}} \log \Big( \frac{i+ a'/r_i}{i+\alpha'}\Big).
\end{align*}

Step 1 uses Lemma \ref{lem:gamma2}. Step 2 assumes $n_{ij} > 0$, and uses properties of the logarithm function. Step 3 uses Lemma \ref{lem:gamma}. The result follows as desired.
\end{proof}
\begin{customcor}{6D}
\cite{cussens2015gobnilp} Given that $r_i^{+} := |\{j: n_{ij} > 0\}|$, then 
\begin{equation*}
         \sum_{j=1}^{r_{\Pi_i}} \log \frac{\Gamma( \alpha')}{\Gamma(n_{ij} + \alpha')} +  \sum_{k=1}^{r_i} \log \frac{\Gamma(n_{ijk} + \frac{\alpha'}{r_i})}{\Gamma( \frac{\alpha'}{r_i})}  \leq -r_i^{+} \log r_i.
\end{equation*}
\label{cor:gamma3}
\end{customcor}

\begin{proof}
 If $n_{ij} > 0$, then 
 \begin{align*}
     \sum_{i=0}^{n_{ij}} \log \Big( \frac{i+ a'/r_i}{i+\alpha'}\Big) &= - \log r_i \sum_{i=1}^{n_{ij}} \log \Big( \frac{i+ a'/r_i}{i+\alpha'}\Big)\\
     &\leq -\log r_i.
 \end{align*}
Note that as $r_i \geq 2$, and $\alpha' > 0$, it is clear that $i + \alpha'/r_i < i+ \alpha'$. This means that each term in $\sum_{i=1}^{n_{ij}} \log \Big( \frac{i+ a'/r_i}{i+\alpha'}\Big)$ is negative.  This gives us the second inequality. The result then follows from Theorem \ref{thm:gamma4} as desired.
\end{proof}
\begin{customcor}{6E}
\cite{cussens2015gobnilp}
Given a vertex variable $V_i$ and candidate parent sets $\Pi_i$ and $\Pi_i'$
such that $\Pi_i \subset \Pi_i'$ and $\Pi_i \neq \Pi_i'$, let
$r_i^{+}(\Pi_i')$ be the number of positive counts in the contingency table for $\Pi_i'$. If $\score{}{\Pi_i} < r_i^{+}(\Pi_i') \log r_i$ then $\Pi_i'$ and the supersets of $\Pi_i'$ can be safely pruned.
\end{customcor}

% \begin{proof}
%  Let $G'$ be a Bayesian network where $\Pi_i'$ or one of its supersets is a parent set for $V_i$.  Let $G$ be another Bayesian network where $\Pi_i$ is the parent set for $V_i$. 
 
%  Consider the LHS of Corollary \ref{cor:gamma3}. It is the local BDeu score for a parent set $\Pi_i'$ which has $r_{\Pi_i}$ counts $n_{ij}$ in its contingency table and counts $n_{ijk}$ in the contingency table for $\Pi_i' \cup \{V_i\}$, where $\alpha' = \alpha/r_{\Pi_i}$ for some ESS $\alpha$. If $r_i^{+}$($\Pi_i') \log r_i > \score{}{\Pi_i}$ then $\score{}{\Pi_i}$ is
% lower than the local BDeu score for $\Pi_i'$ due to Corollary \ref{cor:gamma3}. Take a candidate parent set $\Pi_i''$. If $\Pi_i' \subset \Pi_i''$ then $r_i^{+}(\Pi_i'')  \leq r_i^{+}(\Pi_i')$ and so $r_i^{+}(\Pi_i'') \log r_i \leq r_i^{+}(\Pi_i') \log r_i$, as $r_i \geq 2$. From this
% it follows that the local score for $\Pi_i''$ must also be more than $\score{}{\Pi_i} $. The result
% follows as desired.
% \end{proof}
 The corollary can be generalized to a new pruning rule.
\begin{customthm}{6}
Given a vertex variable $V_i$ and candidate parent sets $\Pi_i$ and $\Pi_i'$ such that $\Pi_i \subset \Pi_i'$ and $\Pi_i \neq \Pi_i'$, let
$r_i^{+}(\Pi_i')$ be the number of positive counts in the contingency table for $\Pi_i'$. If $\score{}{\Pi_i} + \epsilon < r_i^{+}(\Pi_i') \log r_i$, for some $\epsilon \in \mathbb{R^+}$ then $\Pi_i'$ and the supersets of $\Pi_i'$ can be safely pruned.
\label{thm:jamesrelaxed}
\end{customthm}

\begin{proof}
 Let $G'$ be a Bayesian network where $\Pi_i'$ or one of its supersets is a parent set for $V_i$.  Let $G$ be another Bayesian network where $\Pi_i$ is the parent set for $V_i$. 
 
Consider the LHS of Corollary \ref{cor:gamma3}. It is the local BDeu score for a parent set $\Pi_i'$ which has $r_{\Pi_i}$ counts $n_{ij}$ in its contingency table and counts $n_{ijk}$ in the contingency table for $\Pi_i' \cup \{V_i\}$, where $\alpha' = \alpha/r_{\Pi_i}$ for some ESS $\alpha$. If $r_i^{+}$($\Pi_i') \log r_i > \score{}{\Pi_i} + \epsilon$ then $\score{}{\Pi_i} + \epsilon$ is
lower than the local BDeu score for $\Pi_i'$ due to Corollary \ref{cor:gamma3}. Take a candidate parent set $\Pi_i''$. If $\Pi_i' \subset \Pi_i''$ then $r_i^{+}(\Pi_i'')  \leq r_i^{+}(\Pi_i')$ and so $r_i^{+}(\Pi_i'') \log r_i \leq r_i^{+}(\Pi_i') \log r_i$, as $r_i \geq 2$. From this
it follows that the local score for $\Pi_i''$ must also be more than $\score{}{\Pi_i} + \epsilon $. Using Lemma \ref{lem:scoreprune}, the result follows as desired.
\end{proof}

\end{document}